\documentclass[journal, twoside]{IEEEtran}

\ifCLASSINFOpdf
\else
\fi

\usepackage{cite}
\usepackage{float}

\usepackage[T1]{fontenc}

\usepackage{graphicx}
\usepackage{lmodern}
\usepackage{booktabs} 
\usepackage{adjustbox}
\usepackage{stfloats}
\usepackage{wrapfig}
\usepackage{lipsum}
\usepackage{amssymb}
\usepackage{amsmath}
\usepackage{siunitx}
\usepackage{mathtools}
\usepackage{dsfont}

\usepackage{multirow}

\usepackage{makecell}

\let\labelindent\relax
\usepackage[inline]{enumitem}
\usepackage[ruled,vlined,linesnumbered,noend]{algorithm2e}
\usepackage{ifthen}
\usepackage{dsfont}
\usepackage{balance}
\usepackage{dirtytalk}
\usepackage{afterpage}
\usepackage{bm}
\usepackage{times}
\usepackage{enumitem}
\usepackage[colorlinks=true, urlcolor=blue, linkcolor=red]{hyperref}
\usepackage{cleveref}

\newcommand{\N}{\mathbb{N}}
\newcommand{\ra}{\rightarrow}
\newcommand{\R}{\mathbb{R}}

\newcommand{\method}{PolyMerge}  

\newtheorem{remark}{Remark}
\newtheorem{theorem}{Theorem}

\makeatletter
\newcommand{\crefnames}[3]{%
  \@for\next:=#1\do{%
    \expandafter\crefname\expandafter{\next}{#2}{#3}%
  }%
}
\makeatother
\crefnames{section}{Sec.}{}
\crefnames{equation}{Eq.}{}
\crefnames{figure}{Fig.}{}
\crefnames{theorem}{Thm.}{}

\DeclareMathOperator*{\argmin}{arg\,min}

\newcommand{\defeq}{\vcentcolon=}
\newcommand{\bigmid}{\text{ }\big|\text{ }}

\newcommand{\gsmodel}{\mathcal{M}}
\newcommand{\gsellipsoid}{\mathcal{M}}
\newcommand{\ellipsoid}{\mathcal{E}}
\newcommand{\Chi}{\mathcal{X}}
\newcommand{\Chisafegs}{\mathcal{X}_{\text{safe}}^{\text{GS}}}
\newcommand{\Chiunsafegs}{\mathcal{X}_{\text{unsafe}}^{\text{GS}}}

\newcommand{\hull}{\text{Hull}}
\newcommand{\vol}{\text{Vol}}

\newcommand{\occmap}{O}
\newcommand{\allvoxels}{\mathcal{V}}
\newcommand{\occvoxels}{\mathcal{V}^+}
\newcommand{\gridsize}{g}
\newcommand{\voxcover}{\mathcal{C}}
\newcommand{\voxsubcover}{\overline{\mathcal{C}}}

\newcommand{\chullset}{\mathcal{H}}

\newcommand{\hullsubcover}{\overline{\mathcal{H}}}

\newcommand{\voxsubcovern}[1]{\overline{\mathcal{C}^{#1}}}
\newcommand{\voxsubcoverni}[2]{\overline{\mathcal{C}^{#1}_{#2}}}
\newcommand{\hullsubcovern}[1]{\overline{\mathcal{H}^{#1}}}
\newcommand{\hullsubcoverni}[2]{\overline{\mathcal{H}^{#1}_{#2}}}

\newcommand{\arbset}{S}

\newcommand{\aabb}{\mathcal{B}}

\newcommand{\start}{x_\text{start}}
\newcommand{\goal}{x_\text{goal}}
\newcommand{\startm}[1]{x_\text{start}^{#1}}
\newcommand{\goalm}[1]{x_\text{goal}^{#1}}
\newcommand{\Deltat}{\Delta t}
\newcommand{\nominalu}{u_\text{nom}}

\newcommand{\cbfu}{u_\text{cbf}}
\newcommand{\cbfut}{u_{\text{cbf}, t}}

\newcommand{\clf}{V(x)}

\title{
\method{}: Compressing 3D Gaussian Splats with Polytope Coverings for Provably Safe Resource-Constrained
Navigation
}

\usepackage{fancyhdr}
\fancypagestyle{firstpage}{
  \fancyhf{}
  \fancyhead[LE,RO]{\footnotesize \thepage}
  \fancyhead[LO]{\footnotesize IEEE ROBOTICS AND AUTOMATION LETTERS. PREPRINT VERSION. ACCEPTED April, 2026}
  
}

\author{Jihoon Hong$^{1}$, Chih-Yuan Chiu$^{1}$, Sara Fridovich-Keil$^{1}$, and Glen Chou$^{1}$
\thanks{$^{1}$Georgia Institute of Technology.}
}

\begin{document}

\maketitle

\thispagestyle{firstpage}
\pagestyle{fancy}

\begin{abstract}

Obstacle avoidance is essential for safe navigation and motion planning. 
Recent radiance field reconstruction methods
enable
object detection and modeling with 
high
fidelity, but remain too memory- and compute-intensive for
on-board perception-based path planning.
To address these limitations, we propose \method{} to 
convert
a large, photorealistic 3D Gaussian Splatting (3DGS) model of a scene into a lightweight representation 
of convex polytopes whose 
union
provably over-approximates
all obstacles in the original 3DGS model.
\method{} 
tunes the
polytope count
to trade off conservativeness and compute cost, and integrates with
control barrier functions (CBFs) to 
plan collision-free paths.
We showcase \method{} in simulation and 
hardware experiments
on a 
Crazyflie drone, which uses \method{} to compute and follow safe trajectories in real time 
under severe
onboard compute constraints, outperforming baselines in speed while guaranteeing safety.
For our code and videos, visit \url{https://athlon76.github.io/PolyMerge-website/}.
\end{abstract}

\begin{figure*}[!b]
  \centering
  \vspace{-13pt}
  \includegraphics[width=1.0\textwidth]{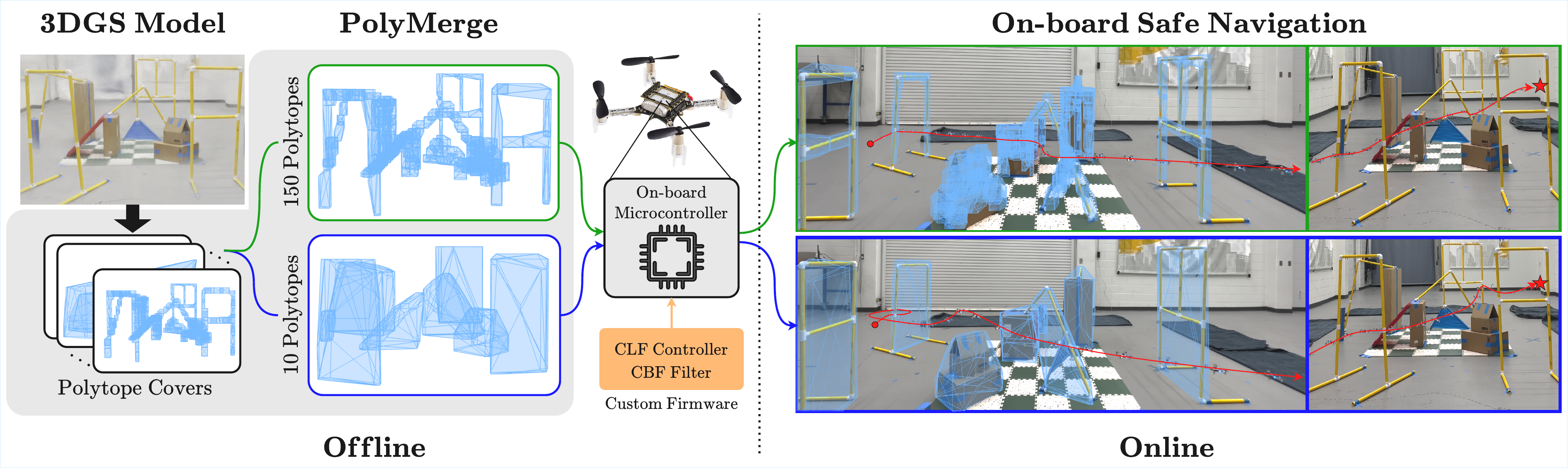}
  \vspace{-8mm}
  \caption{\textbf{Overview.} \textbf{(Left)} \method{} converts a 3DGS model into multiple polytope covers, each of which covers all obstacles but with different numbers of convex hulls and levels of conservativeness. \textbf{(Middle)} A cover that satisfies the hardware-driven computational constraints can be selected and loaded to the microcontroller of the drone, \textbf{(Right)} which then performs CBF filtering of nominal control inputs provided by a CLF-based controller to navigate safely, entirely on-board. Using fewer polytopes reduces hardware cost but may introduce conservativeness in navigation, as can be seen from detours of the drone around hoop obstacles. Red circles and stars respectively denote the start and goal positions of the drone.}
  \label{fig:teaser}
\end{figure*}

\section{Introduction}
\label{sec: Introduction}

\IEEEPARstart{T}{o}
guarantee safe operation, a robot must accurately detect obstacles in its environment. As such, recent advances in radiance field-based scene reconstruction \cite{mildenhall2020nerf, YuFridovichkeil2021Plenoxels, mueller2022instant, Kerbl2023GaussianSplatting3D}, such as 3D Gaussian splatting (3DGS) methods \cite{Kerbl2023GaussianSplatting3D}, have attracted increasing attention in the robotics community for their potential to precisely identify the location and geometry of obstacles. However, the high computational and memory demands of most existing scene reconstruction approaches preclude their direct integration into perception pipelines on resource-constrained robot hardware.
For example, a 3DGS model often contains hundreds of thousands to millions of Gaussians, occupying megabytes or gigabytes of memory. This far exceeds the memory capacity of mobile platforms, typically on the order of kilobytes for small drones such as Crazyflies.
Moreover,
existing approaches to improve the computational efficiency of 
scene reconstructions often focus on retaining visual fidelity for vision tasks \cite{HansonTu2025PUP3DGS, Guedon2023SuGaR}, but do not ensure geometric over-approximation of occupied space and thus cannot safely be used for motion planning. Among methods designed for robotic perception, despite improvements in computational efficiency \cite{ChenSchwager2025SaferSplat} there remains a gap in performance requirements that prohibits deployment on resource-constrained robotic hardware such as small drones.

To bridge this gap, we propose \method{}, a novel framework for compressing a precomputed 
scene representation to enable on-device real-time obstacle avoidance and motion planning via a control barrier function (CBF)-based safety filter. 
Our method can take any photorealistic scene model as input; our experiments are built on 3DGS \cite{Kerbl2023GaussianSplatting3D}, a state-of-the-art computer vision method that represents a scene as a collection of ellipsoids. \method{} first extracts a volumetric occupancy grid to ensure all obstacles in the 
input scene are considered occupied. It then computes a covering set of convex polytopes which contains all occupied space.
Finally,
\method{} performs iterative merging by repeatedly replacing pairs of convex polytopes with their convex hull, 
with the merged polytopes chosen to minimize the volume of unoccupied space within the convex hull. This iterative merging allows \method{} to simplify a scene to accommodate arbitrary computational constraints while ensuring obstacle over-approximation with  
small conservativeness, enabling safe and efficient navigation. Our contributions include:
\begin{enumerate}
    \item We present \method{}, a scene compression algorithm for constructing and merging polytope set covers that compute guaranteed over-approximations of 3D obstacles in a given 3DGS environment model.
    \item We show that \method{} can modulate the number of geometric primitives used to describe obstacles
    to satisfy strict resource constraints at the cost of increased conservativeness in the resulting obstacle over-approximation. 
    \item In simulations and hardware experiments on resource-constrained Crazyflies, we illustrate that \method{} outperforms existing baselines \cite{ChenSchwager2025SaferSplat, Guedon2023SuGaR} (which cannot be deployed at all on Crazyflies or do not guarantee safe navigation) in efficiently generating safe trajectories.
\end{enumerate}
\section{Related Work}
\label{sec: Related Work}

\subsection{Scene Reconstruction}
\label{subsec: Scene Reconstruction for Robotics Tasks} 

The design of efficient algorithms to generate high-fidelity 3D scene reconstructions from multi-view 2D images has long been a core focus in both computer vision and robotics.
Recent, key advances include the development of implicit neural radiance fields \cite{mildenhall2020nerf}, plenoptic voxel grids \cite{YuFridovichkeil2021Plenoxels}, and state-of-the-art 3D Gaussian Splatting (3DGS) \cite{Kerbl2023GaussianSplatting3D} scene models which are the input to our method.
While 3DGS captures photorealistic detail of obstacles in a scene, a 3DGS model can contain hundreds of thousands of individual ellipsoids that render it too computationally intensive for direct use in robotic perception and safe navigation.
To improve computational efficiency, many methods have been proposed to reduce scene complexity, chiefly by pruning away ellipsoids not essential to visual fidelity \cite{HansonTu2025PUP3DGS, Ali2024TrimmingtheFat}, or estimating object surfaces via a 2D mesh \cite{Guedon2023SuGaR} that is more efficient for downstream computation. While 
\cite{HansonTu2025PUP3DGS, Ali2024TrimmingtheFat, Guedon2023SuGaR}
reduce 3DGS model complexity, they 
do not ensure safety for downstream robotics tasks, as space occupied by ellipsoids in the original 3DGS model may be marked unoccupied in a compressed or mesh-extracted representation.

\subsection{CBFs and 
Safe Control from Perception Data}
\label{subsec: CBFs and Safety Filters}

\looseness-1Safety filters
are designed to minimally adjust unsafe control inputs to prevent safety violations.
While a diverse range of safety filters have been proposed \cite{HsuHuFisac2024TheSafetyFilter}, we focus on Control Barrier Function (CBF)-based methods 
\cite{AmesCoogan2019CBFs}, 
which enforce forward set invariance to guarantee safety.
In particular, many recent works present CBF-based controller designs which incorporate perception data for obstacle detection and avoidance \cite{Sa2024PointCloudBasedCBFRegression, Long2021LearningBarrierFunctionsWithMemoryforRobustSafeNavigation, Unlu2023CBFLiDARInertialOdometry, Zhou2024ControlBarrierAidedTeleoperationVisualInertialSLAM}. 
For instance, \cite{Sa2024PointCloudBasedCBFRegression} and \cite{Long2021LearningBarrierFunctionsWithMemoryforRobustSafeNavigation} use CBFs to avoid obstacles which are currently visible to the robot's sensors, such as LiDAR or stereo cameras, but do not attempt to fully map the robot's environment. 
In contrast, \cite{Unlu2023CBFLiDARInertialOdometry} and \cite{Zhou2024ControlBarrierAidedTeleoperationVisualInertialSLAM} formulate CBFs based on scene-wide obstacle representations constructed using LiDAR-inertial odometry and visual-inertial simultaneous localization and mapping (SLAM), respectively, but do not take advantage of high-fidelity photorealistic scene models such as 3DGS \cite{Kerbl2023GaussianSplatting3D}. 
Other methods perform safe perception-based control using image data \cite{
dawson2022learning, chou2022safe
}, but focus on ensuring safety under localization error caused by a learned perception module while assuming simple, known obstacle geometry (e.g., a single spherical obstacle); our method complements these approaches.

Methodologically,
since our experiments use the polytope-based obstacle set representations produced by our \method{} algorithm to synthesize CBF-based controllers, our work is also related to \cite{Molnar2025NavigatingPolytopesWithSafetyACBFApproach}, which likewise generates CBFs in environments with polytope-shaped obstacles. However, in contrast to \cite{Molnar2025NavigatingPolytopesWithSafetyACBFApproach}, our main contribution lies in constructing lightweight, polytope-shaped obstacle representations from 3DGS models, rather than the downstream design of CBF controllers.

\subsection{Safe Motion Planning using Scene Maps}
\label{subsec: Safe Motion Planning using Scene Maps}

The use of radiance field-based map construction to guide safe motion planning tasks has received increasing attention in the robotics literature in recent years.
In particular, prior work has used neural \cite{Adamkiewicz2022NerfNavVisionOnlyRobotNavigationinaNeuralRadianceWorld, KurenkovNFOMP} or Gaussian representations \cite{Chen2025GRaDNavPlusPlus} to encode obstacle locations within a \textit{collision loss term}; such methods, however, cannot guarantee collision avoidance.
In contrast, \cite{Low2025SousVide} and \cite{Chen2025SplatNav} likewise use 3DGS-based scene reconstructions to encode obstacle locations and geometries, but enforce collision avoidance guarantees via \textit{hard constraints}, similar to our work.
However, whereas \cite{Low2025SousVide} and \cite{Chen2025SplatNav} use model predictive control and trajectory-level collision checking, respectively, to enforce collision avoidance, our \method{} method uses CBF-based filters to design minimally-invasive safe control inputs and can be run on-board with very limited compute (e.g., on Crazyflie drones).
Methodologically, our work is most similar to \cite{Tong2023NerfCBF, ChenSchwager2025SaferSplat}, which likewise formulate a CBF safety filter 
over
a 3DGS scene map, improving reconstruction fidelity compared to odometry- and SLAM-based representations.
However, our \method{} method attains higher computation and memory efficiency compared to \cite{Tong2023NerfCBF, ChenSchwager2025SaferSplat}, by merging polytope set covers over a 3DGS scene to over-approximate obstacles.
\section{Preliminaries}
\label{sec: Preliminaries}

\subsection{3D Gaussian Splatting (3DGS)}
\label{subsec:prelim_3dgs}

A 3DGS model $\gsmodel = \{\ellipsoid_i\}_{i \in I}$ consists of ellipsoids $\ellipsoid_i$ derived from a number of Gaussian distributions indexed by $I$, each specified by a mean vector $c_i \in \R^3$, positive definite covariance matrix $\Sigma_i \in \R^{3 \times 3}$, an opacity value, and a set of spherical harmonics coefficients to describe color \cite{Kerbl2023GaussianSplatting3D}.
Concretely, each ellipsoid $\ellipsoid_i$ is centered at the corresponding mean vector $c_i$ with kernel given by $\Sigma_i^{-1}$, i.e., $\ellipsoid_i := \{p \in \R^3: (p-c_i)^\top \Sigma_i^{-1}(p-c_i) \leq 1 \}$ for each $i \in I$.

\subsection{Control Barrier Functions (CBFs)}
\label{subsec:prelim_cbf}

CBFs enable the synthesis of controllers that guarantee robot safety by rendering a provided safe set forward invariant \cite{AmesCoogan2019CBFs}.
Formally, consider control-affine system dynamics $\dot x = f(x) + g(x) u$ evolving over a state space $\Chi \subseteq \R^n$, with controls $u$ drawn from a set $\mathcal{U} \subseteq \R^{n_i}$ of admissible controls, where $f: \Chi \ra \R^n$ and $g: \Chi \ra \R^{n \times n_i}$.
Given a \textit{safe set} $\mathcal{S} \subseteq \Chi$ characterized as the zero super-level set of a function $\psi: \Chi \ra \R$, i.e., $\mathcal{S} = \{x \in \mathcal{X}: \psi(x) \geq 0 \}$, we wish to ensure 
that
$x(t) \in \mathcal{S}$ for all $t\in[0, T]$.
We call $\psi$ 
a CBF if there exists an extended class-$\mathcal{K}_{\infty}$ function $\alpha:\R\rightarrow\R$ satisfying, $\forall x' \in \Chi$, $\exists u' \in \mathcal{U}$ \cite{AmesCoogan2019CBFs}:
\begin{align} 
\label{eq:cbf_condition_affine}
&
\big[L_f \psi(x') + L_g \psi(x') u' \big] \geq-\alpha\big(\psi(x')\big),
\end{align}
where $L_f$ and $L_g$ denote Lie derivatives with respect to $f$ and $g$, respectively.
A controller which generates inputs $u(t) \in \mathcal{U}$ satisfying (at each $t$) 
\eqref{eq:cbf_condition_affine} with $x' = x(t)$ and $u' = u(t)$, would guarantee the forward-invariance of $\mathcal{S}$ with respect to 
$x(t)$. 
In other words, if $x(0) \in \mathcal{S}$ and at each $t$, $u(t) \in \mathcal{U}$ satisfies \eqref{eq:cbf_condition_affine} with $x = x(t)$, then $x(t) \in \mathcal{S}$ for all $t$.

\subsection{CBF in Polytope Scenes}
\label{subsec:prelim_cbf_polytope}
While finding a valid CBF that satisfies \eqref{eq:cbf_condition_affine} is often challenging, a previous work \cite{Molnar2025NavigatingPolytopesWithSafetyACBFApproach} established a systematic method of constructing the CBF when obstacles are represented with convex polytopes under single-integrator dynamics $\dot x = u$.
Specifically, it assumes a finite collection of obstacles $\chullset=\{H_k\}_{k\in K}$ with index set $K$,
where each $H_{k}$ is a convex polytope described by bounding planes $\Lambda_k$:
\begin{equation}
\label{eq:polytope}
    H_{k}=\textstyle\bigcap_{l\in \Lambda_k}\big\{ p\in\R^3 \mid 
    n_{kl}^\top (p-w_{kl}) \leq 0 \big\},
\end{equation}
Above, $n_{kl} \in \R^3$ and $w_{kl} \in \R^3$ denote the normal and offset vectors of the $l$-th plane of $H_{k}$,  respectively.
Then, the barrier function, defined as
\begin{equation}
\label{eq:cbf_polytope}
    \psi(p)=\min_{k\in K}\max_{l\in \Lambda_k} n_{kl}^\top (p-w_{kl}),
\end{equation}
evaluates to a strictly positive value if and only if $p$ does not belong to any polytope $H_k$.

Since both the min and max functions contained in $\psi(p)$ are non-differentiable, they are approximated by the smooth soft-min and soft-max functions, respectively defined as $A \mapsto -\frac{1}{\kappa}\log\big(\sum_{a\in A}{e^{-\kappa a}}\big)$ and $A \mapsto \frac{1}{\kappa}\log\big(\sum_{a\in A}e^{\kappa a}\big) \ \forall A \subset \R$,
which admit Lie derivatives; we use $\kappa=100$ in our experiments.
\cite{Molnar2025NavigatingPolytopesWithSafetyACBFApproach} then illustrates that with a function $\alpha(\psi(x)) := \beta \psi(x)$ for some $\beta > 0$, $\psi$ is a CBF that renders the zero super-level set $\mathcal{S}$ (which by construction does not coincide with the interior of any polytope) forward-invariant for a point agent.
\cite{Molnar2025NavigatingPolytopesWithSafetyACBFApproach} also extends \eqref{eq:cbf_polytope} to a spherical agent with radius $r$ as $\psi_r(p)=\psi(p)-r$,
which
inflates all obstacles by $r$.
\section{Problem Statement}
\label{sec: Problem Statement}

While SaferSplat \cite{ChenSchwager2025SaferSplat} has explored CBF-based safe navigation directly on a 3DGS scene, it cannot be directly used for on-board safe navigation, where both memory and compute resources are often too  limited to process a large set of Gaussians.
\looseness-1To circumvent these challenges, we formulate our \method{} algorithm that compresses a 3DGS model into a lightweight representation, in turn enabling the construction of a CBF that can be evaluated on-device in real time.

For a 3DGS model $\gsellipsoid$, let the set of unsafe states defined by the model $\Chiunsafegs$ be
\begin{equation}
    \Chiunsafegs=\textstyle\bigcup\limits_{\ellipsoid_i\in \gsellipsoid}\ellipsoid_i,
\end{equation}
\looseness-1which includes all object surfaces implicitly defined by the ellipsoids $\ellipsoid_i$.
We consider $\Chiunsafegs$ and $\Chisafegs\defeq \big(\Chiunsafegs\big)^c$, as the ground truth unsafe and safe sets, respectively.
We also define 
$\vol(\cdot)$ to return the volume of a given subset of $\R^3$.

\looseness-1We now describe the problem of 3DGS compression for on-board safe motion planning. 
Suppose we are given a 3DGS model $\gsellipsoid$, and a budget $B$ over the number of polytopes we can process on-device. We aim to obtain a collection $\chullset^\star$ of convex polytopes
which solves \eqref{Eqn: Compression, Optimization Program}, as given below:
\begin{subequations} 
\label{Eqn: Compression, Optimization Program}
\vspace{-1mm}
\begin{align} \label{eq:objective}
    \min_{\chullset} \hspace{5mm} &L(\chullset)  \defeq \vol\left(\bigcup_{H_k\in \chullset} H_k \Big\backslash \Chiunsafegs\right), \\ \label{eq:c1_overapprox}
    \text{s.t.} \hspace{5mm} &\Chiunsafegs \textstyle\subseteq \bigcup_{H_k\in \chullset} H_k, \\ \label{eq:c2_budget}
    &|\chullset| \leq B,
\end{align}
\end{subequations}
where $|\cdot|$ denotes set cardinality.
\eqref{eq:c1_overapprox} guarantees that $\chullset$ over-approximates the true unsafe set and thus ensures safety, while \eqref{eq:c2_budget} limits the number of polytopes in $\chullset$ to respect computational and storage constraints imposed by hardware.
Under \eqref{eq:c1_overapprox} and \eqref{eq:c2_budget}, we aim to minimize \eqref{eq:objective}, which measures how conservative $\chullset$ is compared to $\gsellipsoid$ in describing obstacles.
A set of convex polytopes $\chullset$ which yields a large objective value $L(\chullset)$ would
designate an excessively large subset of $\Chisafegs$ as impassable,
and thus induce overly conservative motion plans downstream.
\section{Methods}
\label{sec:Methods}

\begin{figure*}[!t]
  \centering
  \includegraphics[width=0.95\textwidth]{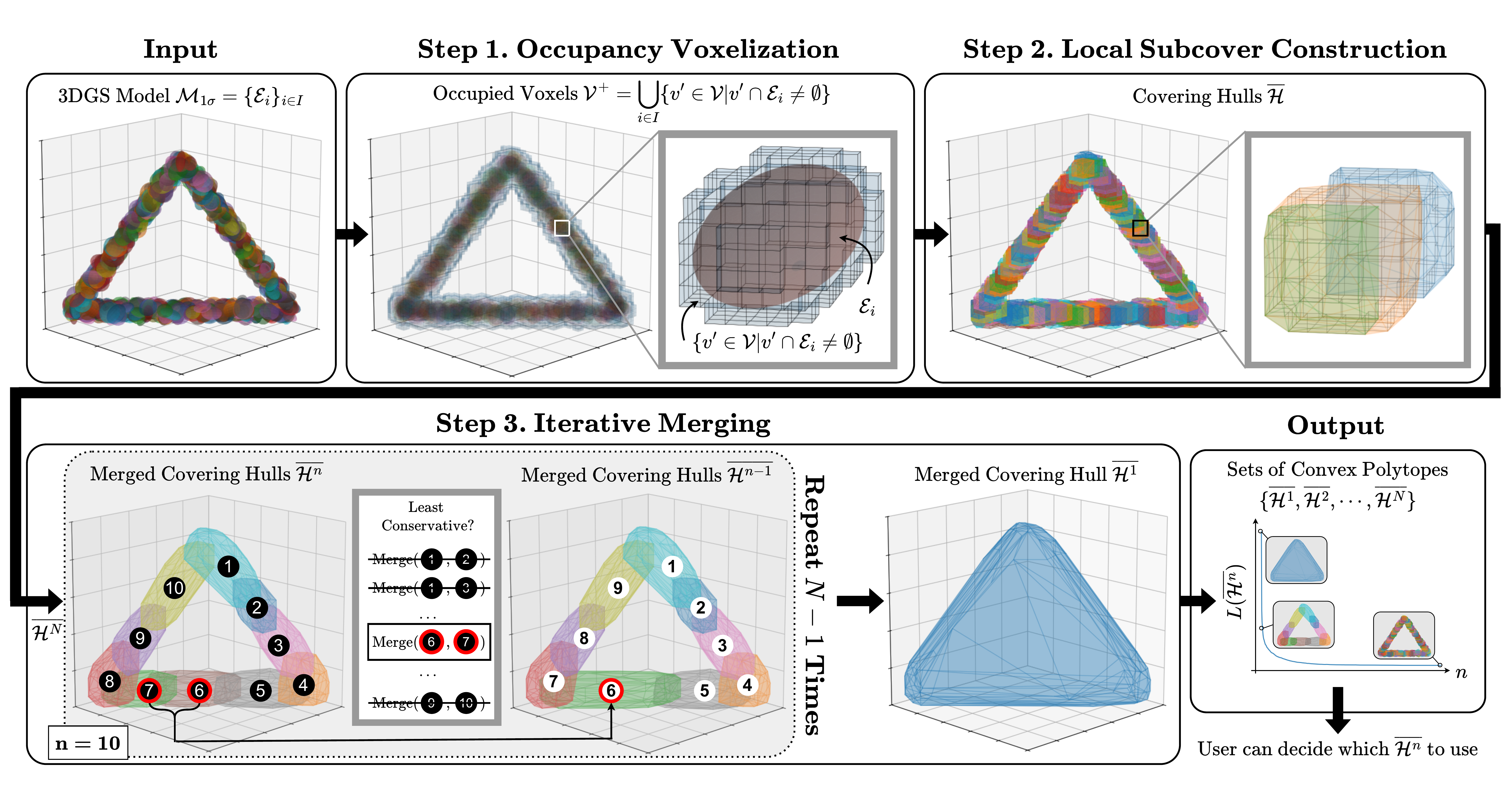}
  \vspace{-5mm}
  \caption{\textbf{
  Compression Algorithm}: Given a 3DGS model, \method{} extracts occupancy information into a voxel grid (\Cref{subsec:occupancy_grid}), constructs a local sub-cover of the occupied voxels (\Cref{subsec:local_subcover}), and iteratively merges the cover elements to incrementally reduce the number of convex hulls (\Cref{subsec:iterative_merging}). 
  }
  \vspace{-5mm}
  \label{fig:method}
\end{figure*}

\par
Since the optimization problem described in \cref{sec: Problem Statement}
is difficult to solve exactly, we present an 
algorithm, \method{}, to compute an approximation $\Tilde{\chullset}$ of $\chullset^\star$ given a 3DGS model $\gsellipsoid$ (\Cref{fig:method}). \method{} proceeds in three steps: (a) Deriving a voxel-based occupancy grid $\occvoxels$ from the cloud of 3DGS ellipsoids (\Cref{subsec:occupancy_grid}); (b) Constructing 
a collection of convex sets $\hullsubcover$ which covers\footnote{Given a set $S$, a collection of sets $\mathcal{T}$ is called a \textit{cover} of $S$ if, for each point $x \in S$, there exists some set $T \in \mathcal{T}$ such that $x \in T$.
  }
the occupied voxels  (\Cref{subsec:local_subcover}); (c) Iteratively and greedily merging local groups of occupied voxels and taking their convex hulls to produce compressed 
collections of polytopes (\Cref{subsec:iterative_merging}). As a result, we obtain provable overapproximations of obstacle sets with varying degrees of conservativeness, as shown in \cref{fig:method}.

\subsection{Occupancy Voxelization}
\label{subsec:occupancy_grid}

\par
First, we extract a voxel-based occupancy grid $\occvoxels$ from the 3DGS model $\gsellipsoid$,
which already achieves some compression because many ellipsoids often overlap and can be covered by a small number of occupied voxels.
We define a voxel $v \subset \R^3$ with center $p_v \in \R^3$ and grid size $g$ by:
\begin{equation}
\label{eq:voxel}
v:= \Big\{p' \in \R^3\mid \|p' -p_v\|_{\infty}\leq\frac{\gridsize}{2} \Big\},
\end{equation}
and let $\allvoxels$ denote the set of all voxels in $\R^3$ whose 8 corners all have coordinates that are multiples of $g$.

We aim to construct a sparse voxel occupancy grid $\occvoxels \subseteq \allvoxels$ from the 3DGS model $\gsellipsoid$ which maintains only voxels that overlap with some ellipsoid $\ellipsoid_i$ in $\gsellipsoid$, i.e.,
\begin{equation}
\label{eq:sparse_voxel}
    \occvoxels\defeq\{v\in\allvoxels \mid \exists \hspace{0.5mm} i \in I \text{ s.t. } v \cap \ellipsoid_i \ne \phi \}.
\end{equation}
To construct $\occvoxels$, it suffices to evaluate the following \textit{occupancy map} $\occmap:\allvoxels\rightarrow\{0,1\}$ over all $v \in \allvoxels$: 
\begin{equation}
\label{eq:occupancy}
    \occmap(v) =  \mathds{1}\big[ \exists \hspace{0.5mm} i \in I \text{ s.t. } v \cap \ellipsoid_i \ne \phi \big],
\end{equation}
\looseness-1since $\occvoxels = \{v \in \allvoxels \mid O(v) = 1 \}$.
In turn, for a given voxel $v$ and an ellipsoid $\ellipsoid_i$, we can determine if $v \cap \ellipsoid_i \ne \phi$ by evaluating whether $\min_{p\in v} (p-c_i)^\top \Sigma_i^{-1} (p-c_i) \leq 1$,
which we can compute with a small quadratic program.
However, na\"ively evaluating whether $v$ intersects every possible $\ellipsoid_i$ is computationally heavy, especially when many ellipsoids and/or voxels are used in the scene description, as is the case when the scene spans a large area, or the grid resolution is high.
To reduce computational cost, we leverage the fact that each voxel often intersects with only a small number of ellipsoids, which allows us to avoid solving a majority of the optimization problems.
Concretely, we construct axis-aligned bounding boxes $\aabb_i$ containing each $\ellipsoid_i$, and, for each voxel $v$, identify a subset of ellipsoids which may overlap with $v$, as indexed by $I_v := \{i \mid v\cap \aabb_i\neq\emptyset\}$. We then evaluate:
\begin{equation}
\label{eq:occupancy_filtered}
\occmap(v)=\mathds{1}\Big[\min_{i\in I_v}\min_{p\in v} (p-c_i)^\top \Sigma_i^{-1} (p-c_i) \leq1\Big].
\end{equation}
After this step, we have a sparse set of occupied voxels $\occvoxels$ covering all obstacles in the original 3DGS scene.
Note that, if an input scene were provided in any format other than 3DGS, \method{} would first extract an occupancy grid $\occvoxels$ and would then proceed as usual.

\subsection{Local Subcover Construction}
\label{subsec:local_subcover}

Although the sparse voxel grid $\occvoxels$ constructed in \cref{subsec:occupancy_grid} covers the obstacle sets, the cardinality of $\occvoxels$ may be too high to satisfy the hardware limitations underlying the constraint \eqref{eq:c2_budget}.
To address this issue, we use $\occvoxels$ to construct a collection $\hullsubcover$ of convex sets, with lower cardinality than $\occvoxels$, which provably covers but does not excessively over-approximate the unsafe space $\Chiunsafegs$ 
(in light of the objective \eqref{eq:objective} and constraint \eqref{eq:c1_overapprox}). 

Concretely, for each $v\in\occvoxels$, we first define the cluster $\voxcover_v \subseteq \occvoxels$ of voxels within a $bg$ neighborhood of $v$:
\begin{equation}
\label{eq:cover_element}
    \voxcover_v = \{v'\in\occvoxels \mid \|p_{v'}-p_v\|_{\infty}\leq b\gridsize\},
\end{equation}
where $b \in \N$ is a design choice and $g$ is the voxel side length as in \eqref{eq:voxel}. 
Let $\voxcover := \{\voxcover_v \mid v \in \occvoxels \}$ be the set of all such voxel clusters, which covers $\occvoxels$ by construction. 
To facilitate the construction of a cover for $\Chiunsafegs$ with lower cardinality than $\occvoxels$, we apply the weighted set cover algorithm \cite{chvatal1979greedy} (Remark \ref{Remark: Weighted Set Cover Algorithm}) to extract a sub-cover $\voxsubcover$ of $\occvoxels$ from $\voxcover$, i.e., a sub-collection $\voxsubcover \subseteq \voxcover$ which covers $\occvoxels$. Finally, we define $\hullsubcover$ to be the collection of convex hulls of voxel clusters in $\voxsubcover$, i.e., we set $\hullsubcover := \{\hull(\voxsubcover_v) \mid \voxsubcover_v \in \voxsubcover \}$. By construction, $\hullsubcover$ covers $\occvoxels$ (thus satisfying \eqref{eq:c1_overapprox}) while including minimal volume outside of $\occvoxels$, 
up to the approximation error of the weighted set cover algorithm.

\begin{remark}(Application of the Weighted Set Cover Algorithm \cite{chvatal1979greedy}.)
\label{Remark: Weighted Set Cover Algorithm}
Let a set $S$, a cover $\mathcal{T}$ of $S$ consisting of subsets of $S$, 
and a weight function $W: \mathcal{T} \ra \R$ be given. The weighted set cover algorithm of \cite{chvatal1979greedy} produces a sub-cover $\overline{\mathcal{T}}$ which approximates the optimal subcover with respect to the original cover $\mathcal{T}$ and the weight function $W$, given by:

\vspace{-13pt}
\begin{equation}\nonumber
\label{eq:setcover_opt}
    \mathcal{T}^\star\defeq\argmin_\mathcal{T'}\Bigg\{ \sum_{\mathcal{T}_i'\in\mathcal{T}'}W(\mathcal{T}_i')\bigmid{ }\mathcal{T}'\subset \mathcal{T},\arbset \subseteq \bigcup\limits_{\mathcal{T}_i'\in\mathcal{T}'}\mathcal{T}_i'\bigg\},
\end{equation}
which is NP-complete to compute \cite{chvatal1979greedy}. In the context of \cref{subsec:local_subcover}, to compute the sub-cover $\voxsubcover$, we take $S = \occvoxels$ and $\mathcal{T} = \voxcover$, and define the weight function $W: \voxcover \ra \R$ as: 
\begin{align} 
\label{eq:setcover_W}
    W(\voxcover_v)\defeq\vol\Big(\hull(\voxcover_v) \Big\backslash \bigcup \{v' \mid v' \in \occvoxels \} \Big).
\end{align}
Note that $W$ approximates the volume of unoccupied space $\allvoxels \backslash \occvoxels$ 
included in the hull of $\voxcover_v$, which quantifies the extent of obstacle set over-approximation.
\end{remark}

\subsection{Iterative Merging}
\label{subsec:iterative_merging}
\par
The final step of \method{} is to refine $\voxsubcover$ so that $\hullsubcover := \{\hull(\voxsubcover_v) \mid \voxsubcover_v \in \voxsubcover \}$ satisfies the hardware constraints \eqref{eq:c2_budget}.
We iteratively merge pairs of elements $\voxsubcover_v$ of the sub-cover $\voxsubcover$ until \eqref{eq:c2_budget} is satisfied.
This process facilitates an explicit visualization of the relation between the number of polytopes and the objective function $L$ (see \eqref{eq:objective}), 
which allows users to navigate the tradeoff between hardware capacity constraints and obstacle over-approximation conservativeness.

Let $N := |\voxsubcover|$, and 
let the sub-cover and convex hulls before merging be respectively denoted by:
\begin{equation}
\label{eq:mergepair_init}
\voxsubcovern{N} \defeq \voxsubcover = \{\voxsubcoverni{N}{i}\}_{i \in I^0}
\text{,\quad }
\hullsubcovern{N} \defeq \hullsubcover = \{\hullsubcoverni{N}{i}\}_{i \in I^0},
\end{equation}
whose elements are now indexed from a new index set $I^0$,
instead of using a voxel $v$.
We reuse the weight $W$ in \eqref{eq:setcover_W} and select a pair $(i^\star,j^\star)$ with minimum weight, as follows:
\begin{equation}
\label{eq:merge}
(i^\star,j^\star)=\argmin_{i,j \in I^0}
W(\voxsubcoverni{N}{i}\cup\voxsubcoverni{N}{j}) .
\end{equation}
Then, we merge the voxel clusters indexed $i^\star$ and $j^\star$, i.e., $\voxsubcoverni{N}{i^\star}$ and $\voxsubcoverni{N}{j^\star}$, into the following new sub-cover (see also \say{Step 3} in Fig. \ref{fig:method}):
\begin{equation}
\label{eq:mergedcover}
\voxsubcovern{N-1}\defeq\Big(\voxsubcovern{N}\setminus\{\voxsubcoverni{N}{i^\star},\voxsubcoverni{N}{j^\star}\}\Big)\cup\{\voxsubcoverni{N}{i^\star}\cup\voxsubcoverni{N}{j^\star} \}
\end{equation}
indexed by $I^1$, and a new set of convex hulls
\begin{equation}
\label{eq:mergedhull}
\hullsubcoverni{N-1}{i} \defeq \{\hull(\voxsubcoverni{N-1}{i})\}_{i \in I^1}.
\end{equation}
We repeat the above process
to acquire $\voxsubcovern{n}$ and $\hullsubcovern{n}$ for $n=N-2, N-3,\cdots{},1$, with $|\hullsubcovern{n}|$ decreasing by one per iteration.

The memory and compute complexities of each iteration as described above grow quadratically with the sub-cover size, so the merging process can become costly when the initial sub-cover size is large.
To avoid incurring this quadratic cost, we construct an axis-aligned bounding box $\aabb_i^n$ for each sub-cover element $\voxsubcoverni{n}{i}$, and evaluate $W(\voxsubcoverni{n}{i} \cup\voxsubcoverni{n}{j})$ 
only when
$\aabb_i^n\cap \aabb_j^n\neq\emptyset$.
We thus only consider merging sub-cover element pairs
$(\voxsubcoverni{n}{i}, \voxsubcoverni{n}{j})$ that are 
in each other's vicinity, since, for 
all other
pairs, $\hull(\voxsubcoverni{n}{i} \cup \voxsubcoverni{n}{j})$ is likely to contain safe set regions, and thus the weight $W(\voxsubcoverni{n}{i}, \voxsubcoverni{n}{j})$ is likely large.

\subsection{Guarantees of Obstacle Set Overapproximation}
\label{subsec: Guarantees of Obstacle Set Overapproximation}

Since the steps of our \method{} method generate progressively conservative overapproximations of the unsafe set $\Chiunsafegs$, the input 3DGS model $\gsellipsoid$ is provably covered by each of the polytope covers $\hullsubcovern{n}$. For a mathematically rigorous argument, see \cref{Thm: Guarantees of Obstacle Set Overapproximation} in the appendix.
\section{Experiments}
\label{sec: Experiments}

\begin{figure*}[!t]
  \centering
  \vspace{1mm}
  \hspace{-5mm}
  \includegraphics[width=0.9\textwidth]{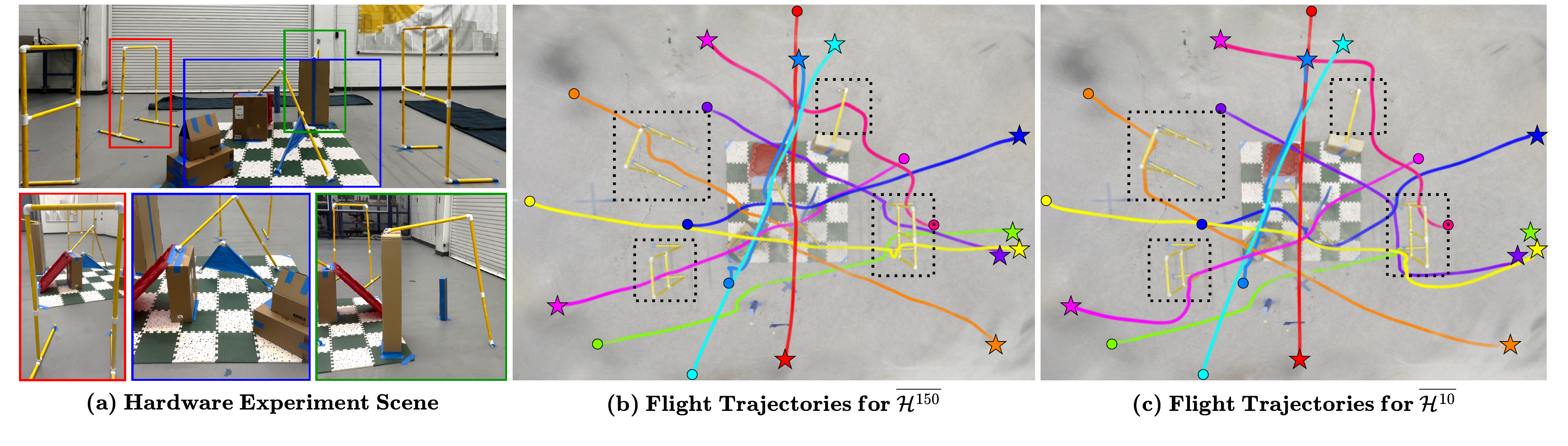}
  \vspace{-6mm}
  \caption{\looseness-1\textbf{Hardware Experiments.} \textbf{(a)} \textit{Top}: The scene in which we perform drone experiments. \textit{Bottom}: 3 different parts of the scene from different views. \textbf{(b-c)} Visualization of 10 different trajectories of the drone, when the obstacles are represented using (b) 150 and (c) 10 convex hulls, respectively (see \cref{fig:teaser}). Circles represent the start positions, and stars denote the goal positions. Dotted boxes highlight the key differences between the two plots, where the drone successfully passes through the hoop in (b) while taking a detour around the hoop (orange trajectory) or even failing to navigate further (green trajectory) in (c).}
  \vspace{-1mm}
  \label{fig:scene_trajectories}
\end{figure*}

We evaluate \method{} through both simulation and hardware experiments.
We describe details of the simulation (\Cref{subsec:simulation}) and hardware (\Cref{subsec:hardware}) experiments, followed by baselines that we compare against and metrics that we report (\Cref{subsec:baselines_and_metrics}), and our experiment results (\Cref{subsec:results}).

\subsection{Simulation}
\label{subsec:simulation}

We perform simulations using the mip-NeRF360 dataset \cite{Barron2022mipNeRF360}, which contains 9 scenes, 5 outdoor and 4 indoor, and is widely used to evaluate 3D scene reconstruction methods.
For each scene, we
apply \method{} to $\gsellipsoid$ to acquire $\{\hullsubcovern{n}\}_{n=1}^N$.
After compression, we sample 100 random start and goal state pairs $\{(\startm{m},\goalm{m})\}_{m=1}^{100}$ from the scene such that $\goalm{m}$ is at least a certain distance away from $\startm{m}$, and all start and goal states belong to $\Chisafegs$.

Then, we simulate navigation from $\startm{m}$ to $\goalm{m}$ in $\hullsubcovern{n}$ for various choices of $n$ with a discrete time-step $\Deltat=0.01$.
For simplicity,
we assume a spherical 3D agent of radius $r$ with single-integrator dynamics 
$x_{t+\Delta t} = x_t + \cbfut \Delta t$, with the control input $\cbfut$ computed as follows.
First, we compute a nominal control input $\nominalu$ via:
\begin{subequations} 
\label{eq:nominal_clf}
\begin{align}
    \nominalu := \argmin_{u} \ &\textstyle\frac{1}{2}\|u\|_2^2 \\
    \text{s.t.} \ &L_fV(x)+L_gV(x)u\leq-cV(x),
\end{align}
\end{subequations}
\looseness-1where $\clf := \frac{1}{2}\|x-\goal\|_2^2$ is a control Lyapunov function (CLF) \cite{AmesCoogan2019CBFs, Romdlony2014UnitingCLFAndCBF}, for some constant $c>0$,
directing the agent towards the goal state.
Second, since the nominal controller is unaware of the obstacles in the scene, to prevent collision, $\nominalu$ is filtered by the CBF $\psi_r$ defined in Sec. \ref{subsec:prelim_cbf_polytope},
which encodes $\hullsubcovern{n}$ as the set of polytopes representing obstacles:
\begin{subequations}
\label{eq:cbf_filter}
\begin{align} 
    \cbfu := \argmin_{u} \ &\textstyle\frac{1}{2}\|u-\nominalu\|_2^2 \\
    \text{s.t.} \ &L_f\psi_r(x)+L_g\psi_r(x)u\geq-\beta\psi_r(x) .
\end{align}
\end{subequations}
Finally, we update the agent state using 
$\cbfu$, i.e., $x_{t+\Deltat}=x_{t}+\cbfu\Deltat.$
From $\start$, we unroll the dynamics following this state update until either 
\begin{enumerate*}[label=(\alph*)]
  \item the agent is sufficiently close to $\goal$, or
  \item the agent stays near a position with little or no state update for a certain number of time-steps, which occasionally occurs due to the myopic nature of the nominal controller and safety filter.
\end{enumerate*}
We repeat this simulation for the same set of start and goal pairs across varying values of $n$.

\subsection{Hardware Experiments}
\label{subsec:hardware}

We also perform hardware experiments on a Crazyflie2.1+ drone in a custom indoor scene (\Cref{fig:scene_trajectories}).
Prior to flight, we train a 3DGS model offline from video footage of the scene, and acquire $\{\hullsubcovern{n}\}_{n=1}^N$ similarly to our simulation experiments.
After occupancy voxelization, we filter out voxels that overlap with less than 30 unique ellipsoids.
This step was included to remove ``floaters", which are Gaussians floating around in free, unoccupied space -- common artifacts found in 3DGS models \cite{feifloaters}.
This filtering could be skipped entirely if the 3DGS model were free of floaters, which is an active area of research \cite{xiong2025sparsegs, wang2025stablegs}.
Then, we modify the drone firmware to run the CLF nominal controller and the CBF-based safety filter on-board during flight.
Specifically, the controller computes $\nominalu$ via \eqref{eq:nominal_clf} with the drone's state estimates, which are produced by its extended Kalman filter (EKF) based on the real-time position information broadcast from a Vicon motion-capture system.
The controller reads off $\hullsubcovern{n}$ stored in the flash memory of the drone, applies the safety filter to $\nominalu$ following \eqref{eq:cbf_filter}, and sets velocity way-points.
While on-board state estimation is also desirable, we leave this to future work and assume exact position information is available in our experiments.

\looseness-1We fly the drone with this custom controller for 10 different $(\start, \goal)$ pairs, with start states on the ground (i.e. $z=0$), and goal states chosen such that obstacles lie along the drone's trajectory.
We repeat this experiment for $n \in \{10,30,50,100,150\}$ and configure
the drone to run the controller at different time-step frequencies for different values of $n$, as the compute cost of the CBF filter increases with $n$.

\vspace{-6pt}
\subsection{Metrics and Baselines}
\label{subsec:baselines_and_metrics}

\begin{table*}[ht]
    \centering
    \caption{\textbf{Memory comparison.} Each row shows the number of ellipsoids or convex hulls needed to represent obstacles, and the file size in memory. The number of hulls is set to be equal for SuGaR + CoACD and \method{} for fair comparison. 
    }
    \vspace{-2mm}
    \label{tab:memory}
    \begin{tabular}{ccccccccccc}
        \toprule
                && \textbf{Bicycle}
                & \textbf{Flowers}
                & \textbf{Garden}
                & \textbf{Stump}
                & \textbf{Treehill}
                & \textbf{Room}
                & \textbf{Counter}
                & \textbf{Kitchen}
                & \textbf{Bonsai}
                \\
        \midrule 
            \multirow{2}{*}{\centering SaferSplat \cite{ChenSchwager2025SaferSplat}
            } 
                & Ellipsoids
                & 1.50M
                & 1.54M
                & 1.61M
                & 1.47M
                & 1.44M
                & 1.06M
                & 1.43M
                & 1.37M
                & 1.40M
            \\
                & Size (MB)
                & 671.1
                & 692.0
                & 718.7
                & 659.8
                & 644.2
                & 473.6
                & 640.3
                & 612.5
                & 626.7
            \\
        \midrule
                SuGaR \cite{Guedon2023SuGaR}
                & Hulls
                & 3,296
                & 2,624
                & 1,006
                & 2,458
                & 2,681
                & 454
                & 1,367
                & 901
                & 2,080
            \\
                +
                & Half-planes in Hulls
                & 189K
                & 149K
                & 58K
                & 150K
                & 150K
                & 26K
                & 76K
                & 53K
                & 119K
            \\
                CoACD \cite{Wei2022ACDFor3DMeshes}
                & Size (MB)
                & 13.5
                & 10.6
                & 4.2
                & 10.7
                & 10.7
                & 1.8
                & 5.6
                & 3.8
                & 8.5
            \\
        \midrule
            \multirow{3}{*}{\centering Ours} 
                & Hulls
                & 3,296
                & 2,624
                & 1,006
                & 2,458
                & 2,681
                & 454
                & 1,367
                & 901
                & 2,080
            \\
                & Half-planes in Hulls
                & 258K
                & 209K
                & 80K
                & 207K
                & 217K
                & 33K
                & 82K
                & 69K
                & 172K
            \\
                & Size (MB)
                & 60.7
                & 49.8
                & 20.1
                & 67.3
                & 58.8
                & 8.5
                & 10.6
                & 17.4
                & 39.4
            \\
        \bottomrule
        \vspace{-6mm}
    \end{tabular}
\end{table*}

\begin{figure*}[ht]
  \centering
  \hspace{-4mm}
  \includegraphics[width=1.0\textwidth]{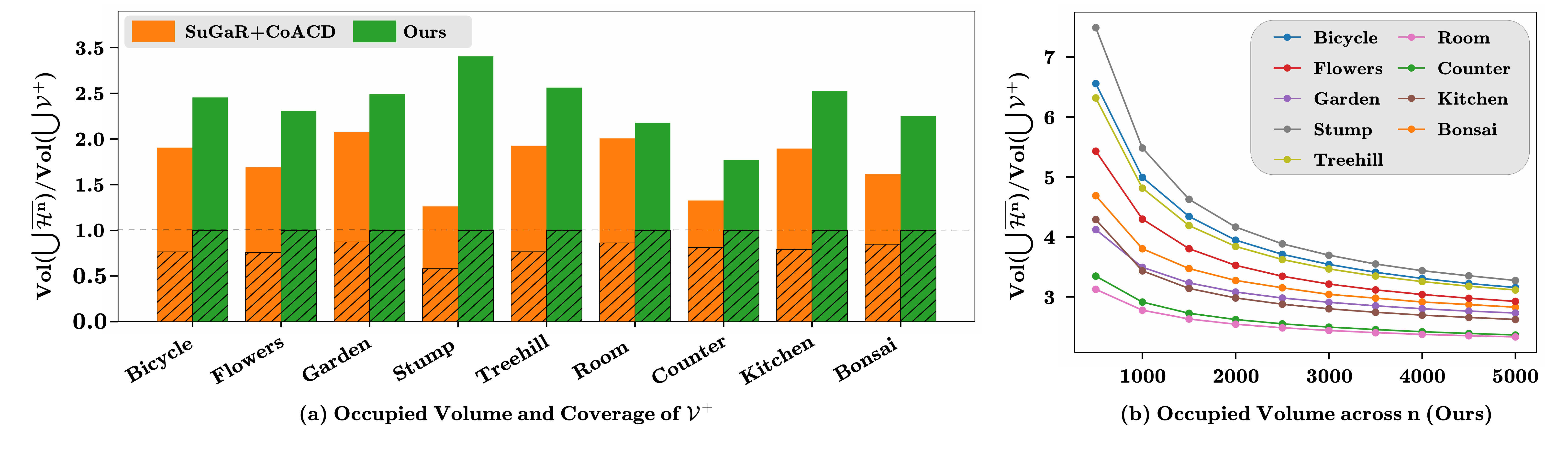}
  \vspace{-4mm}
  \caption{\looseness-1\textbf{
  Conservativeness of Convex Hulls.} (a) The volume of the union of convex hulls, normalized by the volume of the union of occupied voxels. Hatched bars show the volume of the intersection of the two. Hatched bars reaching 1.0 mean all occupied voxels are included in the union of convex hulls, while those below 1.0 mean the convex hulls do not entirely cover all occupied voxels; SuGaR + CoACD fails to cover all occupied voxels. Portions of the bar without hatches are volumes of obstacle-free region included in the hulls, which represent conservativeness. (b) The trade-off between the number of convex hulls $n$ and the conservativeness of the convex hull sub-cover. Using fewer convex hulls results in the hulls covering a larger obstacle-free region.}
  \vspace{-5mm}
  \label{fig:volume}
\end{figure*}

\begin{figure}[ht]
  \centering
  \hspace{-5mm}
  \includegraphics[width=0.5\textwidth]{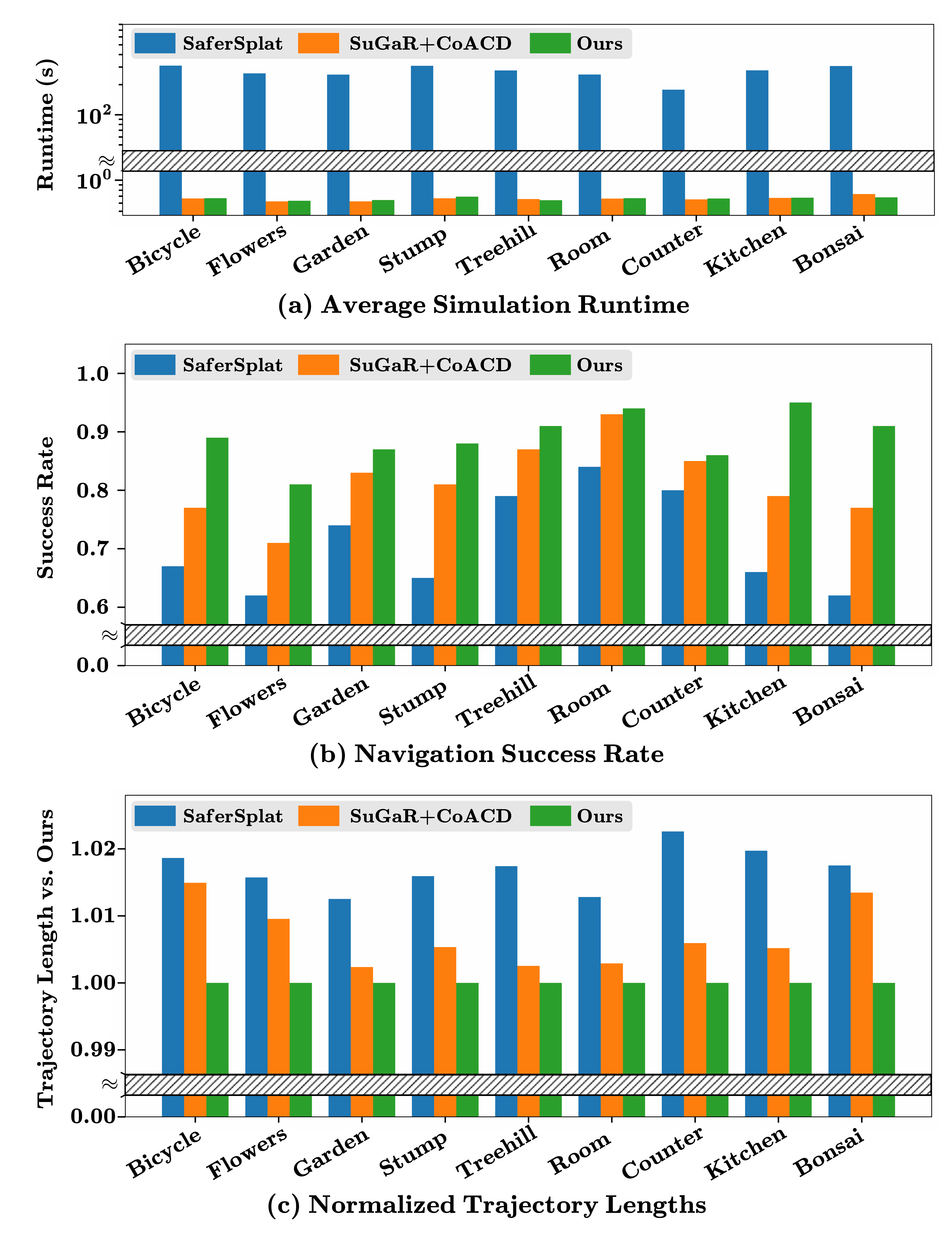}
  \vspace{-5mm}
  \caption{\looseness-1\textbf{Simulation Results.} (a) \textit{Average simulation runtime over 100 trajectories on an RTX A6000 GPU}: Methods using polytopes (SuGaR + CoACD, ours) reduce computation relative to SaferSplat, which uses dense ellipsoid clouds. (b) \textit{Fraction of trajectories that reach within $1\%$ distance to $\goal$}: 
  Our method has the highest success rate.
  (c) \textit{Length of successful trajectories across methods, divided by that of ours and averaged over trajectories}: Our method empirically achieves a shorter trajectory length compared to baselines. (Only trajectories for which all 3 methods succeed are used.)}
  \vspace{-6mm}
  \label{fig:runtime}
\end{figure}

In simulation, we compare our \method{} against two baselines for safe navigation on top of a 3DGS model.
The first baseline is SaferSplat \cite{ChenSchwager2025SaferSplat}, which uses a 3DGS model $\gsellipsoid$ for CBF-based safe navigation.
The second baseline
consists of SuGaR \cite{Guedon2023SuGaR}, a state-of-the-art technique for extracting a mesh from a 3DGS scene, coupled with CoACD \cite{Wei2022ACDFor3DMeshes}, a method that 
converts a mesh into 
a set of convex polytopes.
SuGaR tends to generate 3DGS models with a smaller number of Gaussians, since it introduces 
a
regularization term
to align Gaussians to object surfaces.
For fair comparison, we run all simulations starting from the same 3DGS model trained by SuGaR.

To evaluate the reduction in hardware cost enabled by \method{} relative to our baselines, we report the memory footprint and computation speed of each method.
For memory, we report the total number of ellipsoids and polytopes and the file size in bytes, while for computation speed, we report simulation runtime.
In addition, we quantify the conservativeness of each representation, as measured by the total volume of space occupied by the polytopes, which we approximate through Monte Carlo sampling.
Finally, we report other relevant statistics of the trajectories, including trajectory lengths and goal-reaching success rates.

\subsection{Results}
\label{subsec:results}

\looseness-1\Cref{tab:memory} shows the memory footprint of storing ellipsoids needed for SaferSplat \cite{ChenSchwager2025SaferSplat}, convex hulls acquired from the Gaussians using SuGaR \cite{Guedon2023SuGaR} and CoACD \cite{Wei2022ACDFor3DMeshes}, and finally the convex hulls acquired from \method{} in each scene in mip-NeRF360 \cite{Barron2022mipNeRF360}.
The number of convex hulls for SuGaR + CoACD is chosen as the maximum number attainable by tuning the parameters of CoACD when applying it to the mesh constructed by SuGaR, and we use the same number of convex hulls for \method{} to ensure fair comparison.
Across all scenes, we observe up to 10$\times$ reduction in memory footprint when convex polytopes, rather than ellipsoids, are stored, 
since 3DGS models contain a large number of Gaussians.
In addition, although the same number of convex hulls are used for both polytope-based methods, \method{} occupies more memory than SuGaR + CoACD because its polytopes are composed of more half-planes.
Nevertheless, \method{} provides an over-approximation of the obstacles covering all voxels in $\occvoxels$ while SuGaR + CoACD covers only a strict subset of obstacles, which does not guarantee safety (\Cref{fig:volume}a).
Moreover, \method{} allows users to reduce conservativeness by using more convex hulls (via the trade-off curves in \Cref{fig:volume}b); further improvements can be attained by using a smaller voxel size $\gridsize$ or local cover size $b$.

The smaller number of polytopes compared to ellipsoids also leads to improved compute speeds.
On an NVIDIA RTX A6000 GPU, we observe $\approx$100x speed-up in average simulation runtime in polytope-based representations over ellipsoids (see \Cref{fig:runtime}a).
We also empirically find that the success rate over 100 simulations increases across all scenes under \method{} compared to the baselines (\Cref{fig:runtime}b), and that goal-reaching is achieved more effectively with shorter trajectories (\Cref{fig:runtime}c).
We hypothesize that the over-approximation of the obstacles introduced by the polytopes in \method{} may help regularize the unsafe set defined by the ellipsoids, which makes the agent less susceptible to ``traps" from which the agent can struggle or fail to escape.

Our hardware experiment results are given in \Cref{fig:scene_trajectories} and \Cref{tab:hw_success_cbf}.
We fly the Crazyflie with 
a CLF-based nominal controller and a CBF filter on-board every 1, 2, 2, 3, and 4 stabilizer loops under $\hullsubcovern{n}$, for $n\in\{10,30,50,100,150\}$ respectively.
Since the stabilizer loop of a Crazyflie 2.1+ drone runs at 500Hz 
\cite{Crazyflie2025Documentation}
when 
an EKF is used for pose estimation,
our controller runs at 500, 250, 250, 167, and 125Hz, for each number $n$ of polytopes. These speeds are set at roughly the maximum feasible for the drone's CPU.

\looseness-1The resulting trajectories using 150 and 10 convex hulls, displayed in \Cref{fig:scene_trajectories}b-c, show that the drone is able to navigate through the hoop obstacles when $n=150$, but flies around them when the convex hulls are merged and eventually cover the interior of the hoops.
\Cref{tab:hw_success_cbf} reports additional quantitative results, which include (i) success rate, measured by the fraction of start-goal pairs for which the drone reached within 5 cm of the goal without getting stuck, and (ii) the minimum distance to the set of 150 polytopes representing ground truth obstacles, after visually verifying the tight localization of the polytopes around objects. Since the drone usually flies closer to obstacles when using more polytopes due to less over-approximation, we see a decreasing trend in the second row.
\vspace{-1mm}
\section{Conclusion, Limitations, and Future Work}
\label{sec: Conclusion and Future Work}

We present
\method{}, a method for over-approximating obstacles in a 3DGS scene using a desired number of convex polytopes for safe CBF-based navigation. \method{} paves the way for integrating state-of-the-art photorealistic scene modeling with real-time onboard perception
and collision-free motion planning on highly resource-constrained robotic platforms.
It guarantees safety by ensuring obstacles in a given scene model are over-approximated, while gracefully trading off computational complexity with the degree of over-approximation or conservativeness. 
Future work will address limitations of \method{}, including (i) extension to dynamic scenes with moving obstacles, (ii) incorporating on-board localization, for instance via SLAM or odometry, and (iii) enabling the scene map to update and expand as the robot moves beyond an initial scene model.

\begin{table}[ht]
\vspace{-8pt}
    \centering
    \caption{\textbf{Drone Trajectory Statistics} in hardware experiments.\vspace{-8pt}}
    \label{tab:hw_success_cbf}
    \begin{tabular}{cccccc}
        \toprule
                & \multicolumn{5}{c}{\textbf{Number of Convex Hulls}}
            \\
                & 10
                & 30
                & 50
                & 100
                & 150
            \\
        \midrule 
                \makecell{Success Rate}
                & 0.7
                & 0.8
                & 1.0
                & 0.8
                & 0.9
            \\
                \makecell{Min Obstacle Distance}
                & 0.268
                & 0.253
                & 0.098
                & 0.117
                & 0.097
            \\
        \bottomrule
        \vspace{-9mm}
    \end{tabular}
\end{table}
\bibliography{references}

@article{mueller2022instant,
    author = {Thomas M\"uller and Alex Evans and Christoph Schied and Alexander Keller},
    title = {{Instant Neural Graphics Primitives with a Multiresolution Hash Encoding}},
    journal = {ACM Trans. Graph.},
    issue_date = {July 2022},
    volume = {41},
    number = {4},
    month = jul,
    year = {2022},
    pages = {102:1--102:15},
    articleno = {102},
    numpages = {15},
    publisher = {ACM},
    address = {New York, NY, USA}
}

@article{Kerbl2023GaussianSplatting3D,
author = {Kerbl, Bernhard and Kopanas, Georgios and Leimkuehler, Thomas and Drettakis, George},
title = {{3D Gaussian Splatting for Real-Time Radiance Field Rendering}},
year = {2023},
issue_date = {August 2023},
publisher = {Association for Computing Machinery},
address = {New York, NY, USA},
volume = {42},
number = {4},
issn = {0730-0301},
journal = {ACM Trans. Graph.},
month = jul,
articleno = {139},
numpages = {14}
}

@inproceedings{mildenhall2020nerf,
 title={{NeRF: Representing Scenes as Neural Radiance Fields for View Synthesis}},
 author={Ben Mildenhall and Pratul P. Srinivasan and Matthew Tancik and Jonathan T. Barron and Ravi Ramamoorthi and Ren Ng},
 year={2020},
 booktitle={ECCV},
}

@inproceedings{YuFridovichkeil2021Plenoxels,
      title={{Plenoxels: Radiance Fields without Neural Networks}},
      author={Sara Fridovich-Keil and Alex Yu and Matthew Tancik and Qinhong Chen and Benjamin Recht and Angjoo Kanazawa},
      year={2022},
      booktitle={CVPR},
}

@article{Guedon2023SuGaR,
  title={{SuGaR: Surface-Aligned Gaussian Splatting for Efficient 3D Mesh Reconstruction and High-Quality Mesh Rendering}},
  author={Gu{\'e}don, Antoine and Lepetit, Vincent},
  journal={CVPR},
  year={2024}
}

@inproceedings{Ali2024TrimmingtheFat,
  author    = {Muhammad Salman Ali and Maryam Qamar and Sung-Ho Bae and Enzo Tartaglione},
  title     = {{Trimming the Fat: Efficient Compression of 3D Gaussian Splats through Pruning}},
  booktitle = {BMVC},
  publisher = {BMVA},
  year      = {2024},
}

@InProceedings{HansonTu2025PUP3DGS,
    author    = {Hanson, Alex and Tu, Allen and Singla, Vasu and Jayawardhana, Mayuka and Zwicker, Matthias and Goldstein, Tom},
    title     = {{PUP 3D-GS: Principled Uncertainty Pruning for 3D Gaussian Splatting}},
    booktitle = {CVPR},
    month     = {June},
    year      = {2025},
    pages     = {5949-5958},
}

@INPROCEEDINGS{AmesCoogan2019CBFs,
  author={Ames, Aaron D. and Coogan, Samuel and Egerstedt, Magnus and Notomista, Gennaro and Sreenath, Koushil and Tabuada, Paulo},
  booktitle={2019 18th European Control Conference (ECC)}, 
  title={{Control Barrier Functions: Theory and Applications}}, 
  year={2019},
  volume={},
  number={},
  pages={3420-3431},
}

@article{HsuHuFisac2024TheSafetyFilter,
   author = "Hsu, Kai-Chieh and Hu, Haimin and Fisac, Jaime F.",
   title = "{The Safety Filter: A Unified View of Safety-Critical Control in Autonomous Systems}", 
   journal= "Annual Review of Control, Robotics, and Autonomous Systems",
   year = "2024",
   volume = "7",
   pages = "47-72",
   publisher = "Annual Reviews",
   issn = "2573-5144",
   type = "Journal Article",
  }

@incollection{Unlu2023CBFLiDARInertialOdometry,
author = {Halil Utku Unlu and Dimitris Chaikalis and Vinicius Gonçalves and Anthony Tzes},
title = {{Control Barrier Functions and LiDAR-Inertial Odometry for Safe Drone Navigation in GNSS-Denied Environments}},
booktitle = {Motion Planning for Dynamic Agents},
address = {London},
year = {2023},
editor = {Zain Anwar Ali and Amber Israr},
chapter = {6},
}

@article{Zhou2024ControlBarrierAidedTeleoperationVisualInertialSLAM,
      title={{Control-Barrier-Aided Teleoperation with Visual-Inertial SLAM for Safe MAV Navigation in Complex Environments}}, 
      author={Siqi Zhou and Sotiris Papatheodorou and Stefan Leutenegger and Angela P. Schoellig},
      year={2024},
      eprint={2403.04331},
      journal={arXiv},
}

@INPROCEEDINGS{Sa2024PointCloudBasedCBFRegression,
  author={De Sa, Massimiliano and Kotaru, Prasanth and Sreenath, Koushil},
  booktitle={ICRA}, 
  title={{Point Cloud-Based Control Barrier Function Regression for Safe and Efficient Vision-Based Control}}, 
  year={2024},
  volume={},
  number={},
  pages={366-372},
}

@ARTICLE{Long2021LearningBarrierFunctionsWithMemoryforRobustSafeNavigation,
  author={Long, Kehan and Qian, Cheng and Cortés, Jorge and Atanasov, Nikolay},
  journal={IEEE Robotics and Automation Letters}, 
  title={{Learning Barrier Functions With Memory for Robust Safe Navigation}}, 
  year={2021},
  volume={6},
  number={3},
  pages={4931-4938},
}

@INPROCEEDINGS{Romdlony2014UnitingCLFAndCBF,
  author={Romdlony, Muhammad Zakiyullah and Jayawardhana, Bayu},
  booktitle={53rd IEEE Conference on Decision and Control}, 
  title={{Uniting Control Lyapunov and Control Barrier Functions}}, 
  year={2014},
  volume={},
  number={},
  pages={2293-2298},
}

@INPROCEEDINGS{Molnar2025NavigatingPolytopesWithSafetyACBFApproach,
  author={Molnar, Tamas G.},
  booktitle={IEEE Conference on Control Technology and Applications (CCTA)}, 
  title={{Navigating Polytopes with Safety: A Control Barrier Function Approach}}, 
  year={2025},
  volume={},
  number={}
}

@INPROCEEDINGS{ChenSchwager2025SaferSplat,
  author={Chen, Timothy and Swann, Aiden and Yu, Javier and Shorinwa, Ola and Murai, Riku and Kennedy, Monroe and Schwager, Mac},
  booktitle={ICRA}, 
  title={{A Control Barrier Function for Safe Navigation with Online Gaussian Splatting Maps}}, 
  year={2025},
  volume={},
  number={},
}

@ARTICLE{Chen2025SplatNav,
  author={Chen, Timothy and Shorinwa, Ola and Bruno, Joseph and Swann, Aiden and Yu, Javier and Zeng, Weijia and Nagami, Keiko and Dames, Philip and Schwager, Mac},
  journal={IEEE Transactions on Robotics}, 
  title={{Splat-Nav: Safe Real-Time Robot Navigation in Gaussian Splatting Maps}}, 
  year={2025},
  volume={41},
  number={},
  pages={2765-2784},
}

@article{Wei2022ACDFor3DMeshes,
author = {Wei, Xinyue and Liu, Minghua and Ling, Zhan and Su, Hao},
title = {{Approximate Convex Decomposition for 3D Meshes with Collision-aware Concavity and Tree Search}},
year = {2022},
issue_date = {July 2022},
publisher = {Association for Computing Machinery},
address = {New York, NY, USA},
volume = {41},
number = {4},
issn = {0730-0301},
journal = {ACM Trans. Graph.},
month = jul,
articleno = {42},
numpages = {18},
}

@article{Barron2022mipNeRF360,
    title={{Mip-NeRF 360: Unbounded Anti-Aliased Neural Radiance Fields}},
    author={Jonathan T. Barron and Ben Mildenhall and 
            Dor Verbin and Pratul P. Srinivasan and Peter Hedman},
    journal={CVPR},
    year={2022}
}

@article{chvatal1979greedy,
  title={{A Greedy Heuristic for the Set-covering Problem}},
  author={Chvatal, Vasek},
  journal={Mathematics of operations research},
  volume={4},
  number={3},
  pages={233--235},
  year={1979},
  publisher={INFORMS}
}

@ARTICLE{Chen2025GRaDNavPlusPlus,
  author={Chen, Qianzhong and Gao, Naixiang and Huang, Suning and Low, JunEn and Chen, Timothy and Sun, Jiankai and Schwager, Mac},
  journal={IEEE Robotics and Automation Letters}, 
  title={{GRAD-NAV++: Vision-Language Model Enabled Visual Drone Navigation With Gaussian Radiance Fields and Differentiable Dynamics}}, 
  year={2026},
  volume={11},
  number={2}
}

@inproceedings{chou2022safe,
  title={{Safe Output Feedback Motion Planning from Images via Learned Perception Modules and Contraction Theory}},
  author={Chou, Glen and Ozay, Necmiye and Berenson, Dmitry},
  booktitle={Algorithmic Foundations of Robotics XV},
  year={2022}
}

@article{dawson2022learning,
  title={Learning safe, generalizable perception-based hybrid control with certificates},
  author={Dawson, Charles and Lowenkamp, Bethany and Goff, Dylan and Fan, Chuchu},
  journal={IEEE Robotics and Automation Letters},
  volume={7},
  number={2},
  pages={1904--1911},
  year={2022},
  publisher={IEEE}
}

@ARTICLE{Adamkiewicz2022NerfNavVisionOnlyRobotNavigationinaNeuralRadianceWorld,
  author={Adamkiewicz, Michal and Chen, Timothy and Caccavale, Adam and Gardner, Rachel and Culbertson, Preston and Bohg, Jeannette and Schwager, Mac},
  journal={IEEE Robotics and Automation Letters}, 
  title={{Vision-Only Robot Navigation in a Neural Radiance World}}, 
  year={2022},
  volume={7},
  number={2},
  pages={4606-4613},
  doi={10.1109/LRA.2022.3150497}
}

@ARTICLE{KurenkovNFOMP,
  author={Kurenkov, Mikhail and Potapov, Andrei and Savinykh, Alena and Yudin, Evgeny and Kruzhkov, Evgeny and Karpyshev, Pavel and Tsetserukou, Dzmitry},
  journal={IEEE Robotics and Automation Letters}, 
  title={{NFOMP: Neural Field for Optimal Motion Planner of Differential Drive Robots With Nonholonomic Constraints}}, 
  year={2022},
  volume={7},
  number={4},
  pages={10991-10998},
}

@ARTICLE{Low2025SousVide,
  author={Low, JunEn and Adang, Maximilian and Yu, Javier and Nagami, Keiko and Schwager, Mac},
  journal={IEEE Robotics and Automation Letters}, 
  title={{SOUS VIDE: Cooking Visual Drone Navigation Policies in a Gaussian Splatting Vacuum}}, 
  year={2025},
  volume={10},
  number={5}
}

@INPROCEEDINGS{Tong2023NerfCBF,
  author={Tong, Mukun and Dawson, Charles and Fan, Chuchu},
  booktitle={IEEE International Conference on Robotics and Automation}, 
  title={{Enforcing Safety for Vision-based Controllers via Control Barrier Functions and Neural Radiance Fields}}, 
  year={2023},
  volume={},
  number={}
}

@article{feifloaters,
  title={3d gaussian splatting as new era: A survey},
  author={Fei, Ben and Xu, Jingyi and Zhang, Rui and Zhou, Qingyuan and Yang, Weidong and He, Ying},
  journal={IEEE Transactions on Visualization and Computer Graphics},
  year={2024},
  publisher={IEEE}
}

@article{wang2025stablegs,
  title={{StableGS: A Floater-free Framework for 3D Gaussian Splatting}},
  author={Wang, Luchao and Ren, Qian and Liao, Kaimin and Wang, Hua and Chen, Zhi and Tang, Yaohua},
  journal={arXiv preprint arXiv:2503.18458},
  year={2025}
}

@inproceedings{xiong2025sparsegs,
  title={SparseGS: sparse view synthesis using 3D Gaussian splatting},
  author={Xiong, Haolin and Muttukuru, Sairisheek and Xiao, Hanyuan and Upadhyay, Rishi and Chari, Pradyumna and Zhao, Yajie and Kadambi, Achuta},
  booktitle={2025 International Conference on 3D Vision (3DV)},
  pages={1032--1041},
  year={2025},
  organization={IEEE}
}

@misc{Crazyflie2025Documentation,
    author = {{Crazyflie 2.1+}},
    year = {2025},
	howpublished = {\url{https://www.bitcraze.io/products/crazyflie-2-1-plus/}},
	title = {{Crazyflie 2.1+ Documentation}},
}
\bibliographystyle{IEEEtran}
\appendix
\begin{theorem} \label{Thm: Guarantees of Obstacle Set Overapproximation}
For each $n \in \{1, \cdots, N\}$, $\hullsubcovern{n}$ covers $\Chiunsafegs$.
\end{theorem}

\begin{proof}
First, we prove that $\occvoxels$ covers $\Chiunsafegs := \bigcup_{\ellipsoid_i \in \gsellipsoid} \ellipsoid_i$. Fix $p \in \Chiunsafegs$ arbitrarily.
Since the set of all voxels $\allvoxels$ covers $\R^3$, there exists a voxel $v_p$ which contains $p$, and thus intersects $\ellipsoid_i$. Then, by definition of $\occvoxels$ in \cref{eq:sparse_voxel}, we have $v_p \in \occvoxels$; thus, $\occvoxels$ covers $\Chiunsafegs$. 

Next, we prove that $\hullsubcover := \{\hull(\voxsubcover_v) \mid \voxsubcover_v \in \voxsubcover \}$ covers $\Chiunsafegs$. Again, fix $p \in \Chiunsafegs$. Since $\occvoxels$ covers $\Chiunsafegs$, there exists a voxel $v_p \in \occvoxels$ containing $p$. Since $\voxsubcover$ covers $\occvoxels$, there exists a $\voxsubcover_v \in \voxsubcover$ such that $v_p \in \voxsubcover_v$, so $v_p \subseteq \hull(\voxsubcover_v)$. Thus, $p \in v_p \subseteq \hull(\voxsubcover_v)$, so $\hullsubcover$ covers $\Chiunsafegs$.

Since $\hullsubcovern{N} := \hullsubcover$, to complete the proof, it suffices to prove that if $\hullsubcovern{n+1}$ covers $\Chiunsafegs$ for some $n \in \{1, \cdots, N-1\}$, then $\hullsubcovern{n}$ covers $\Chiunsafegs$. Fix $p \in \Chiunsafegs$ arbitrarily, and let a voxel cluster $\voxsubcoverni{n+1}{i} \in \voxsubcovern{n+1}$, for some $i \in I^{N-n-1}$, be given such that $p \in \hull(\voxsubcoverni{n+1}{i})$. During the process of constructing $\hullsubcovern{n}$ from $\hullsubcovern{n+1}$, either $\voxsubcoverni{n+1}{i}$ will be selected to merge (with $\voxsubcoverni{n+1}{j}$, for some $j \in I^{N-n-1}\backslash \{i\}$), or it will not. In the former case, we have $p \in \hull(\voxsubcoverni{n+1}{i}) \subseteq \hull(\voxsubcoverni{n+1}{i} \cup \voxsubcoverni{n+1}{j})$, with $\voxsubcoverni{n+1}{i} \cup \voxsubcoverni{n+1}{j} \in \voxsubcovern{n}$. In the latter case, we have $p \in \hull(\voxsubcoverni{n+1}{i})$ with $\voxsubcoverni{n+1}{i} \in \voxsubcovern{n}$. Thus, in both cases, there exists a cover set in $\voxsubcovern{n}$ which contains $p$, so $\voxsubcovern{n}$ covers $\Chiunsafegs$.
\end{proof}

\end{document}